\documentclass[letterpaper]{article} 
\usepackage{aaai2026}  
\usepackage{times}  
\usepackage{helvet}  
\usepackage{courier}  
\usepackage[hyphens]{url}  
\usepackage{graphicx} 
\usepackage{amsmath}
\usepackage{amsthm}
\usepackage{amssymb}
 \usepackage{booktabs}
 \usepackage{float} 
\newtheorem{theorem}{Theorem} 
\newtheorem{lemma}{Lemma}

\usepackage{subcaption}
\usepackage{booktabs}
\usepackage{algorithm}
\usepackage{algorithmic}
\usepackage{placeins}
\usepackage{cleveref}
\usepackage{enumitem}
\usepackage{natbib}  
\usepackage{caption} 
\usepackage[caption=false]{subfig} 
\usepackage{newfloat}
\usepackage{listings}
\usepackage{bibentry} 

\DeclareCaptionStyle{ruled}{labelfont=normalfont,labelsep=colon,strut=off}
\floatstyle{ruled}
\newfloat{listing}{tb}{lst}{}
\floatname{listing}{Listing}

\title{Predicting the Performance of Graph
Convolutional Networks with Spectral Properties of
the Graph Laplacian}

\author {
    Shalima Binta Manir\textsuperscript{\rm 1},
    Tim Oates\textsuperscript{\rm 2}
}
\affiliations {
    \textsuperscript{\rm 1}Department of Computer Science, University of Maryland, Baltimore County\\
     \textsuperscript{\rm 2}Department of Computer Science, University of Maryland, Baltimore County\\
    smanir1@umbc.edu, oates@cs.umbc.edu
}

\begin{document}
\nocopyright

\maketitle

\begin{abstract}
A common observation in the Graph Convolutional Network (GCN) literature is that stacking GCN layers may or may not result in better performance on tasks like node classification and edge prediction. We have found empirically that a graph’s algebraic connectivity, which is known as the Fiedler value, is a good predictor of GCN performance. Intuitively, graphs with similar Fiedler values have analogous structural properties, suggesting that the same filters and hyperparameters may yield similar results when used with GCNs, and that transfer learning may be more effective between graphs with similar algebraic connectivity. We explore this theoretically and empirically with experiments on synthetic and real graph data, including the Cora, CiteSeer and Polblogs datasets.  We explore multiple ways of aggregating the Fiedler value for connected components in the graphs to arrive at a value for the entire graph, and show that it can be used to predict GCN performance. We also present theoretical arguments as to why the Fiedler value is a good predictor.
\end{abstract}

\section{Introduction}

Many domains are naturally described by graphs, where nodes correspond to entities with properties and edges correspond to relationships among entities. Beyond the more obvious naturally occurring graphs, like social networks, citation networks, and molecular structures, one can represent text as an attributed graph (a dependency parse tree) and visual scenes as objects with properties that stand in spatial relationships to one another \cite{velivckovic2023everything}.
Common inferential tasks performed on graphs include node prediction (e.g., there should be an object on which the cup is resting), link prediction (e.g., Alice probably knows Bob), and graph classification (e.g., that molecule is highly toxic).

Graph Convolutional Networks (GCNs) \cite{zhang2019graph} have emerged as a powerful tool for representation learning from graph data, with successful applications in physical systems modeling, molecular chemistry, biology, recommendation system design, traffic prediction, and  other practical domains \cite{zhou2020graph}.  There are many forms of GCNs, but the basic idea is message passing.  If each node initially contains a feature vector (though it is possible to have features on edges as well), those vectors are transmitted as messages to immediate neighbors and then all incoming messages are combined at a node to produce a new vector \cite{zhang2019graph}.  Adding layers to a GCN allows messages to spread farther in the graph.  With $k$ layers, each node's influence reaches all other nodes within $k$ hops.  The resulting embeddings support node, edge, and graph classification tasks \cite{xu2023comprehensive}.

Despite their power and popularity, challenges remain with GCNs.  They typically assume static graphs \cite{wu2022improving} and have difficulty scaling to large graphs \cite{mostafa2020permutohedral}.  More relevant for this paper, unlike other forms of deep networks where increasing depth often leads to better representations, increasing depth with GCNs often has the opposite effect.  For example, as the reach of nodes' messages increases with GCN depth, the problem of over-smoothing arises, where the embeddings of all of the nodes become indistinguishable \cite{sun2022over}.  Similarly, over-squashing occurs as information from distant parts of the graph is squeezed (squashed) into a finite sized vector \cite{giraldo2022understanding}.  Adding layers also exacerbates problems with vanishing gradients \cite{li2019deepgcns}.

Collectively, this means that choosing the right number of layers for GCNs is a crucial design step, for which there is little guidance in the literature.  In this paper we show that a graph's algebraic connectivity (known as the Fiedler value or Fiedler eigenvalue), which is the first non-zero eigenvalue of the graph's Laplacian matrix, can inform the choice of the number of GCN \cite{kipf2016semi} layers. 

Fisser et al.\cite{fesser2025performance} introduce the Fiedler \emph{singular} value of a gradient-derived graph to diagnose over-squashing by analyzing gradient flow bottlenecks in trained GNNs. In contrast, our work employs the Fiedler \emph{eigenvalue} of the input graph Laplacian to derive a theoretical lower bound on feature variation resulting from Laplacian smoothing. While both approaches leverage spectral tools, our method is designed to provide insight into GCN behavior \emph{prior} to training, enabling informed decisions about layer depth. This stands in contrast to Fisser’s \emph{post hoc} diagnostic, which is applicable only after gradient-based learning has occurred.

Our contributions are as follows:
\begin{itemize}
    \item Introduction of algebraic connectivity of graphs as a way of characterizing how well GCNs will perform as a function of the number of convolutional layers.
    \item Empirical demonstration of a ``sweet spot'' for the algebraic connectivity that indicates good GCN performance on node classification tasks on both synthetic and real-world data.
    \item Theoretical arguments about why algebraic connectivity is a good predictor of GCN performance.
\end{itemize}
The remainder of this paper is organized as follows.  The next section introduces concepts and notation relevant to graphs, algebraic connectivity, and GCNs.  Following that is a theoretical exploration of algebraic connectivity, and then a series of experiments in which the algebraic connectivity of graphs is used to predict GCN performance. 

\section{Preliminaries}
Let \( G(V,E) \) be a graph \( G \) where the set of vertices is \( V \) and the set of edges is \( E \). The adjacency matrix of \( G \) is \( A \in \{0,1\}^{n \times n} \). The degree matrix \( D \) is defined as:

\[
D_{i,j} =
\begin{cases}
\text{degree}(v_i), & \text{if } i = j; \\
0, & \text{otherwise.}
\end{cases}
\]

The graph Laplacian matrix \( L \) is defined as \(L= D-A\) \cite{grone1990laplacian}, which has an important role in defining convolution in GCNs \cite{daigavane2021understanding}. 

For any square matrix $A$, it is the case that $Av= \lambda v$ where $v$ is an eigenvector of the matrix and $\lambda$ is the corresponding eigenvalue. Every square matrix has at least one eigenvalue \cite{mirchev2016beyond}.  In this work, we will be concerned with the eigenvalues of graph Laplacian matrices.

For any connected graph $G$ with $n$ vertices the Laplacian matrix of $G$ is symmetric and has real eigenvalues $0=\lambda_1\le \lambda_2\le \lambda_3....\le \lambda_n$.
The {\em Fiedler value} is the first nonzero eigenvalue of the Laplacian matrix $L$, i.e., $\lambda_2$ \cite{fiedler1989laplacian}. This value is associated with connectivity, among a few other graph-theoretic properties.

For a graph with $k$ disconnected components, the Laplacian matrix can be represented as a block diagonal matrix where each block corresponds to a single connected component of the graph.  Let $L_i$ denote the graph Laplacian of the $i^{th}$ component.
The spectrum of $L$ is defined by the union of the spectra of the $L_i$ \cite{von2007tutorial}.
Each $L_i$ contains eigenvalue 0 with multiplicity 1, with a corresponding eigenvector ${\bf 1}$, so the number of 0 eigenvalues in $L$ equals the number of connected components in $G$.
For any undirected graph, it is the case that algebraic connectivity is bounded from above by the vertex connectivity, or the number of nodes that must be removed to disconnect the graph \cite{barriere2013fiedler}.

Convolutional neural networks (CNNs) applied to, for example, images use filters that compute weighted sums of values in a (typically) square region of adjacent pixels. GCNs use polynomial filters applied to the graph Laplacian matrix $L$. An example of a polynomial filter is shown below:\newline
\( S(L)=w_0I_n+w_1L+w_2L^2+...+w_kL^K =\sum\limits^{k}_{i=0}{w^i L^i} \)
\\
Here $I$ is the identity matrix, the values of $w_i$ are the learned parameters of the polynomial filter, and $L^i$ denotes raising $L$ to the power $i$ or multiplying $i$ copies of $L$ together \cite{daigavane2021understanding}. There are several types of polynomial filters, with the choice being domain or task specific \cite{gama2020graphs}.

One such filter is the Chebyshev polynomial, which follows De Moivre's formula, meaning it can be generated recursively as follows: $P_k(x)= 2x P_{k-1}(x)-P_{k-2}(x)$ where $P_0(x)=1$ and
$P_1(x)=x$ \cite{defferrard2016convolutional}. They act as low pass filters in graph neural networks \cite{he2022convolutional}.

Among other Graph Neural Networks, GCNs use Chebyshev polynomials for convolution operations \cite{kipf2016semi}. GCN convolution layers aggregate information from the local neighborhood of nodes in a graph using a 1st-order approximation of the Chebyshev polynomial, aggregating information from 1-hop neighbors. For the layer-wise propagation rule, see Appendix A.2.

Although GCNs are an effective type of graph neural network for classification and node/link prediction, they suffer from over-smoothing \cite{elinas2022addressing}. Adding multiple convolution layers causes node features to converge and the resulting model is unable to distinguish among them \cite{rusch2023survey}.

\section{Algebraic Connectivity and GCNs}
This section explores the connection between a graph's Fiedler value (algebraic connectivity) and the operations used by GCNs, with the goal of understanding what the former may indicate about the latter.

\subsection{Estimating the Topological 
Distance of a Graph}

GCN models perform a type of Laplacian smoothing that combines local vertex data with graph topology in convolution layers. Vertices that are related topologically are likely to belong to the same class and smoothing these features increases their similarity, thereby aiding the classification process for GCNs \cite{li2018deeper}.  When nodes from different classes have small topological distances problems with over-smoothing are more likely to arise with deeper GCNs \cite{chen2020measuring}. 
 
The {\em mean distance} is a measure of the topological distance of a graph. It is the average of all distances between distinct vertices of a graph G. Calculating the mean distance of a large and complex graph is computationally expensive, especially for real-world data. Approximating a bound on mean distance is often more practical. Topological distance plays a significant role in GCN performance, and the Fiedler value can be used to estimate the mean distance of a graph. It imposes both upper and lower bounds for a graph $G$ of order $n$ \cite{mohar1991eigenvalues}.

Let $\overline{m}$ denote the mean distance of a graph $G$ and let $\lambda_2$ denote the Fiedler value of the graph. Then a lower bound on the mean distance is:
\[
(n-1)\overline{m}(G) \geq \frac{2}{\lambda_2} + \frac{n-2}{2}.
\]
and an upper bound on the mean distance is:
\[
\overline{m}(G) \leq \frac{n}{n-1} \left(\left\lceil \frac{\Delta + \lambda_2}{4\lambda_2} \ln(n-1)\right\rceil + \frac{1}{2} \right).
\]
Here $\Delta$ is the maximum node degree of graph $G$.

From these bounds we can see that the Fiedler value and the mean distance of a graph have an inverse relationship. Higher mean distances mean, on average, that nodes are less closely connected and lower mean distances indicate closer connections. If the algebraic connectivity is high then the mean distance will be low and the graph will be more susceptible to over-smoothing as layers are added.  High algebraic connectivity is also associated with expander-like graphs (see Appendix A.3).

A similar analysis (see the Appendix A.4) shows that the Fiedler value and the diameter of a graph are inversely related. If we apply too few GCN layers to a \emph{large}-diameter graph, under‑reaching can occur. Conversely, \emph{large}-diameter graphs are also more susceptible to over‑squashing, because exponentially many long‑range messages must be compressed into fixed‑size vectors as depth grows; small‑diameter graphs, in contrast, are more likely to encounter over‑smoothing. High range further exacerbates over‑squashing \cite{alon2020bottleneck}. If a graph’s diameter is large or small, as computed from the Fiedler value, we can adjust GCN depth accordingly.  

\subsection{Possible Optimum Feature Distance}

It should be clear that the degree to which node features converge as a function of the number of rounds of graph convolution is central to the performance of GCNs.  This is complicated by the fact that we'd like node features to converge differenly depending on the classes to which the nodes belong.  What follows is a series of lemmas (with citations when they come from prior work) and a proposed theorem that will shed light on what the Fiedler value tells us about distances between feature vectors after Laplacian smoothing.
we now establish a series of lemmas and conclude with a key theorem that connects the spectral properties of the graph, specifically the Fiedler value, to the distances between node features after Laplacian smoothing. The lemmas progressively build intuition, clarify important concepts like Dirichlet energy and feature distances, and lay the theoretical groundwork required for the main result.
We begin with Lemma~\ref{lem:dirichlet}, a foundational result linking the Dirichlet energy to the total feature distance, providing a clear geometric interpretation of the Laplacian smoothing process.
\begin{lemma}
\label{lem:dirichlet}
Let \( G = (V, E) \) be an undirected graph with degree matrix \( D \), adjacency matrix \( A \), and Laplacian \( L = D - A \). For each vertex \( i \), let \( v_i \in \mathbb{R}^m \) denote its feature vector, and let the feature matrix be \( V = [v_1^{\top} \;\cdots\; v_n^{\top}]^{\top} \in \mathbb{R}^{n \times m} \). Then,
\[
\operatorname{Tr}(V^{\top} L V) = \sum_{\{i, j\} \in E} \|v_i - v_j\|_2^2.
\] 
\end{lemma}
See Appendix A.5 for the proof.

\textbf{Remark}
Lemma~\ref{lem:dirichlet} can be viewed as a straightforward vector–valued generalisation of Marsden’s result \cite[Lemma 3.6]{marsden2013eigenvalues}


\paragraph{Interpretation.}
The scalar \(\operatorname{Tr}(V^{\top}LV)\) is often called the
\emph{Dirichlet (or smoothness) energy} of the graph signal~\(V\)\cite{digio2023energies}.
Lemma~\ref{lem:dirichlet} shows that it is nothing more than the
aggregate squared Euclidean distance between the features of adjacent
vertices, thus providing a direct \emph{feature--distance}
interpretation.
Note, however, that this quantity \emph{measures} variation—it is not
itself a graph-convolution operation; spectral filters and GCN layers
are motivated by \emph{minimising} this very score.


\begin{lemma}
\label{lem:first-eigenpair}
Let \( G = (V, E) \) be an undirected graph with \( |V| = n \), and let \( L \in \mathbb{R}^{n \times n} \) denote its (combinatorial) Laplacian matrix. Then
\[
  \lambda_{1} = 0,
  \qquad
  u_{1} = \frac{1}{\sqrt{n}} \mathbf{1},
\]
where \( \mathbf{1} = [1,\dots,1]^{\top} \in \mathbb{R}^n \), and \( L u_1 = \lambda_1 u_1 \). This is a standard result in spectral graph theory \cite[Sec.~1.2]{Chung1997}.
\end{lemma}

See Appendix A.6 for the proof.

\begin{lemma}
\label{lem:no-constant-mode}
Let \(G=(V,E)\) be a graph with \(n\) vertices and Laplacian
\(L\in\mathbb{R}^{n\times n}\).
Let \(v_i\in\mathbb{R}^{m}\) be the feature vector of vertex \(i\), and define
\(
  \bar V=[\,v_1^{\!\top}\;\dots\;v_n^{\!\top}]^{\top}
  \in\mathbb{R}^{n\times m}.
\)
Expand \(\bar V\) in the orthonormal eigenbasis
\(\{u_1,\dots,u_n\}\subset\mathbb{R}^{n}\) of \(L\):
\[
   \bar V
   \;=\;
   \sum_{i=1}^{n} u_i\,w_i^{\!\top},
   \qquad
   w_i\in\mathbb{R}^{m}.
   \tag{$\ast$}
\]
If the features are \emph{centered}, i.e.,
\(\sum_{i=1}^{n} v_i = 0_m\),
then \(w_{1}=0_m\), and hence
\[
   \bar V
   \;=\;
   \sum_{i=2}^{n} u_i\,w_i^{\!\top}.
\] (See~\cite[Chapter~1]{Chung1997} for background on Laplacian eigenbases and orthogonality.)
\end{lemma}
Proof is provided in Appendix A.7.
\begin{lemma}
\label{lem:frobenius}
\textit{Let \(G\) be a graph with \(n\) vertices and Laplacian
\(L\in\mathbb{R}^{n\times n}\).
Stack the vertex features \(v_i\in\mathbb{R}^{m}\) row-wise in the
matrix \(\bar V=[\,v_{1}^{\top}\dots v_{n}^{\top}]^{\top}
\in\mathbb{R}^{n\times m}\).
Suppose that}
\[
   \bar V
   \;=\;
   \sum_{i=1}^{n} \bar u_i\,w_i^{\!\top},
   \qquad
   w_i\in\mathbb{R}^{m},
\]
\textit{where \(\{\bar u_i\}_{i=1}^{n}\subset\mathbb{R}^{n}\) is an
orthonormal eigenbasis of \(L\).
If the features are centered,
\(\sum_{i=1}^{n} v_i = 0_m\),
then}
\[
  \sum_{i=1}^{n}\|v_i\|_2^{2}
  \;=\;
  \sum_{i=2}^{n}\|w_i\|_2^{2}.
\]
\end{lemma}

See Appendix A.8 for the proof.

\textbf{Remark.}
Lemma~\ref{lem:frobenius} can be viewed as a matrix-valued extension of Parseval’s identity in the Laplacian spectral domain. It relies on the orthonormality of the Laplacian eigenbasis and the fact that centering removes the component along the constant eigenvector. Analogous scalar results are well established in spectral graph theory and graph signal processing~\cite{Chung1997,shuman2013emerging}.

\begin{lemma}
\label{lem:vectorcase}
\textit{Let \(G\) be a graph with \(n\) vertices, Laplacian
\(L\in\mathbb{R}^{n\times n}\) and eigen-pairs
\(\{(\lambda_i,u_i)\}_{i=1}^{n}\) ordered
\(0=\lambda_{1}\le\lambda_{2}\le\dots\le\lambda_{n}\).
Stack the vertex features \(v_i\in\mathbb{R}^{m}\) row-wise in the
matrix \(\bar V=[\,v_{1}^{\top}\dots v_{n}^{\top}]^{\top}
\in\mathbb{R}^{n\times m}\)
and expand it in that eigenbasis}
\[
   \bar V
   =\sum_{i=1}^{n} u_i\,w_i^{\!\top},
   \qquad
   w_i\in\mathbb{R}^{m}.
\]
\textit{If the features are centered
\(\bigl(\sum_{i=1}^{n}v_i=0_m\bigr)\), then \(w_1=0_m\) and}
\[
  \sum_{(i,j)\in E}\!\|v_i-v_j\|_2^{2}
  =\operatorname{Tr}(\bar V^{\top}L\bar V)
  =\sum_{i=2}^{n}\lambda_i\,\|w_i\|_2^{2}.
\]
\end{lemma}
The proof is provided in Appendix A.9.

\textbf{Remark.}
Lemma~\ref{lem:vectorcase} extends the classical spectral decomposition of the Laplacian quadratic form to matrix-valued signals. In the scalar case, this identity expresses the Dirichlet energy as a weighted sum of squared spectral coefficients, with weights given by the Laplacian eigenvalues. See~\cite[Chapter~1]{Chung1997} and~\cite{shuman2013emerging} for foundational treatments in spectral graph theory and graph signal processing.


\begin{theorem}[Fiedler Bound for Minimal Dirichlet Energy]
\label{thm:fiedler-min}
Let \(G=(V,E)\) have Laplacian
\(L\in\mathbb{R}^{n\times n}\) with eigenvalues
\(0=\lambda_{1}\le\lambda_{2}\le\dots\le\lambda_{n}\)
(the eigenvalue \(\lambda_{2}\) is the \emph{Fiedler value}).
For vertex features \(v_i\in\mathbb{R}^{m}\) satisfying
\[
  \sum_{i=1}^{n}\|v_i\|_2^{2}=n,
  \qquad
  \sum_{i=1}^{n}v_i=0_m ,
\]
we always have
\[
  \sum_{(i,j)\in E}\|v_i-v_j\|_2^{2}\;\;\ge\; n\,\lambda_{2},
\]
and the lower bound \(n\lambda_{2}\) (the \emph{Fiedler bound}) is
attained iff all energy lies in the Fiedler mode, i.e.
\[
   w_2=\sqrt{n}\,q,\ q\in\mathbb{R}^{m},\qquad 
   w_3=\dots=w_n=0_m ,
\]
where the \(w_i\) are the coefficients in the expansion
\(\bar V=\sum_{i=2}^{n}u_i w_i^{\!\top}\).
\end{theorem}
See Appendix A.10 for the proof of Theorem 1.
It is important to note that the proposed theorem establishes a connection between the minimum feature distance and the Fiedler value, and comparing the actual feature distances with this minimum can give insight into the effects of adding GCN layers. If the empirical feature distribution is close to the minimum distance for a particular graph, then the graph features are smooth and GCN performance classification tasks should be good. On the other hand, adding layers to the GCN quickly lead to over-smoothing. 

While the variational characterization of the Fiedler value via the Rayleigh quotient is well known in spectral graph theory, our contribution lies in its reinterpretation as a lower bound on the feature variation achievable through Laplacian smoothing. We explicitly frame this as a constraint for GCN feature evolution and propose a normalized diagnostic score $(\rho_k)
$, which allows one to monitor how close a representation is to this theoretical limit. To our knowledge, this is the first application of the Fiedler bound as a dynamic smoothing indicator for GCN layer selection.
\paragraph{Feature–distance score.}
Theorem~\ref{thm:fiedler-min} shows that $n\lambda_{2}$—the
\emph{Fiedler bound}—is the \emph{smallest total feature distance}
any centered, unit-energy signal can achieve on the graph, so we use
$n\lambda_{2}$ as a \emph{gold standard for smoothness}.  For the hidden
representation $V^{(k)}$ produced after the $k$-th GCN layer, define
\begin{equation}
  \rho_{k}
  \;=\;
  \frac{\operatorname{Tr}\!\bigl((V^{(k)})^{\top} L V^{(k)}\bigr)}
       {n\lambda_{2}} .
  \label{eq:rho_k}
\end{equation}

\textbf{Before training ($k=0$).}
With $V^{(0)}\!=V_{\text{in}}$ we have
\begin{equation}
  \rho_{0}
  \;=\;
  \frac{\operatorname{Tr}\!\bigl(V_{\text{in}}^{\top} L V_{\text{in}}\bigr)}
       {n\lambda_{2}} .
  \label{eq:rho_0}
\end{equation}
A large $\rho_{0}$ (say $\rho_{0}\!\gg\!1$) indicates that the raw
features vary sharply across edges, leaving the network ample room to
\emph{smooth} them.

If $\rho_{k}$ stays $\gg 1$ after a few layers, additional depth can
legitimately reduce feature variation; when
$\rho_{k}\!\approx\!1$, further layers are likely to push the model
into the \emph{over-smoothing} regime.


\section{Experiments and Results}

This section describes experiments that seek to establish empirical evidence for the theoretical connection between the Fieldler value and GCN performance described in the previous section. In our experiments, we implemented the Graph Convolutional Network (GCN) architecture originally proposed by Kipf et al.~\cite{kipf2016semi}, closely following their formulation and training methodology.
Our current formulation assumes that label similarity is correlated with topological closeness (homophily), which underlies the use of Laplacian smoothing.

We evaluate our method on two real-world citation network datasets Cora and CiteSeer~\cite{yang2016revisiting}—as well as the PolBlogs dataset, originally introduced by Adamic and Glance~\cite{adamic2005political} and accessed via the Network Repository~\cite{rossi2015network}.
We focus on the citation networks Cora and Citeseer,  
the same benchmarks used in several recent over‑smoothing studies e.g.\ the  
theoretical analysis of mean‑aggregation smoothing by Keriven at el~\cite{keriven2022not}.

Additionally, we also used synthetic datasets.
Nodes in the synthetic datasets belong to one of three classes, with the node features being drawn from one of three normal distributions, each paired with a class.  The means and variances of those normals are as follows: (250, 50), (100, 90), and (400, 200).
Synthetic graphs enable controlled variation of structural parameters such as edge density and connectedness which in turn affect the spectral properties of the Laplacian. Unlike real-world graphs, where many factors (e.g., label distribution, feature heterogeneity) may co-vary, these graphs allow us to isolate the specific influence of algebraic connectivity on GCN performance.
Real-world graphs can have multiple connected components.  The Fielder value, $\lambda_2$, in such cases is 0.  To addrress this we compute $\lambda_2$ for each component and aggregate them using a weighted average, where the weights are the ratios of the number of nodes in each component to the total number of nodes in the graph.  In this way, larger connected components have greater influence on the aggregate Fiedler value that will be used to make predictions about GCN performance.
The number of nodes, edges, and features in the synthetic datasets were varied, and the values of $\lambda_2$ were calculated.  Then GCNs were trained with up to five layers for node classification.
\begin{figure*}[t]
\centering
\begin{minipage}[b]{0.23\textwidth}
    \centering
    \includegraphics[width=\linewidth]{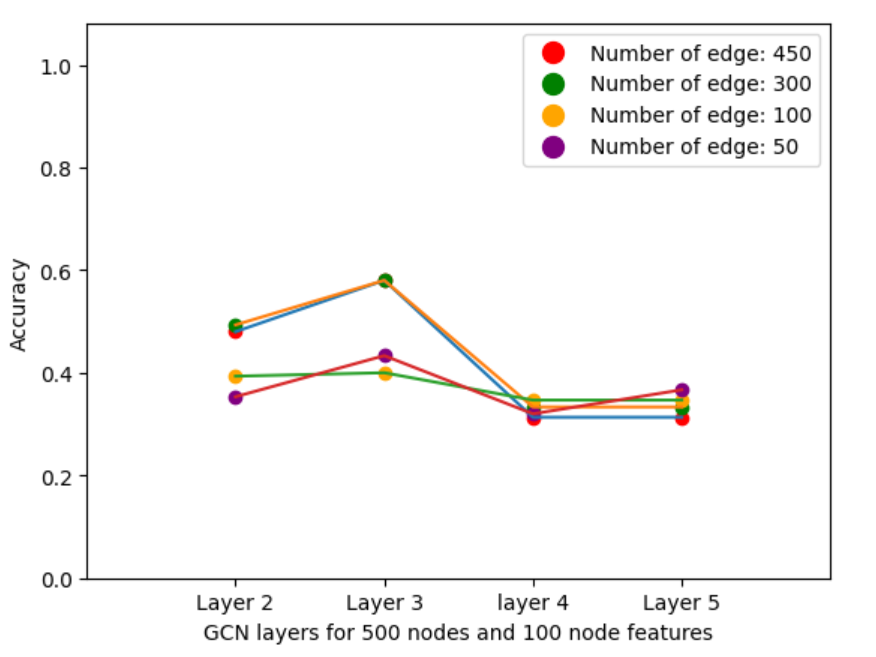}
    \caption*{(a)}
\end{minipage}
\hspace{0.01\textwidth}
\begin{minipage}[b]{0.23\textwidth}
    \centering
    \includegraphics[width=\linewidth]{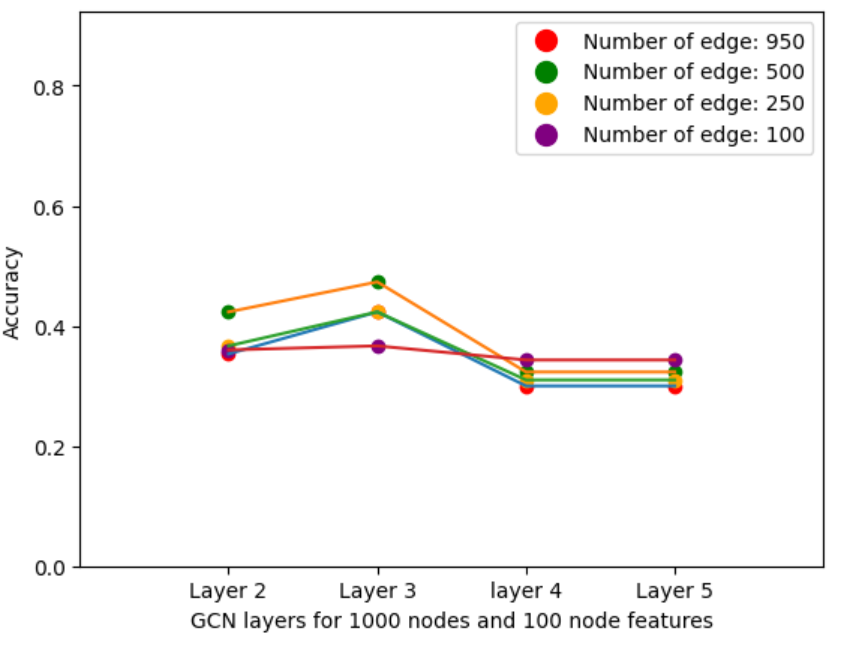}
    \caption*{(b)}
\end{minipage}
\hspace{0.01\textwidth}
\begin{minipage}[b]{0.23\textwidth}
    \centering
    \includegraphics[width=\linewidth]{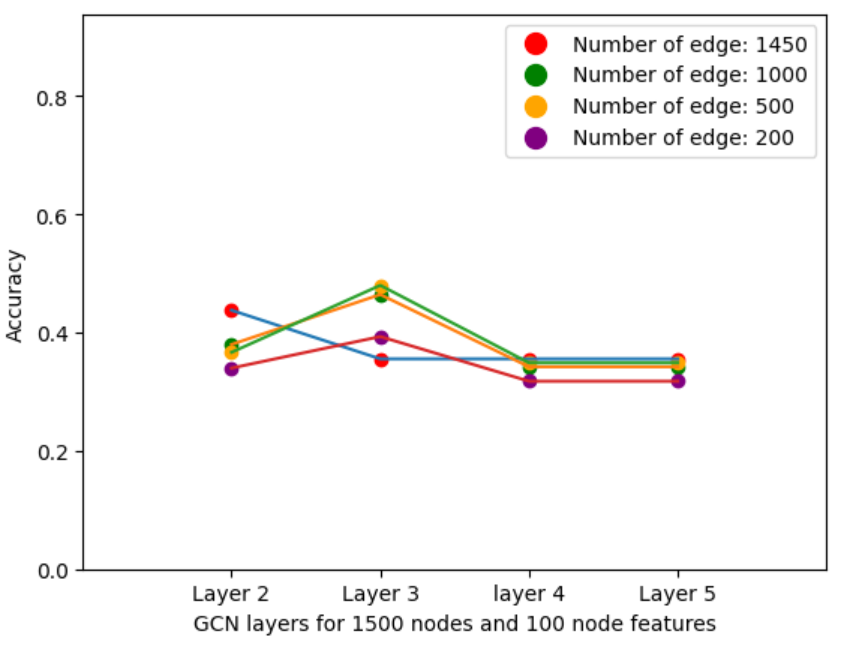}
    \caption*{(c)}
\end{minipage}
\hspace{0.01\textwidth}
\begin{minipage}[b]{0.23\textwidth}
    \centering
    \includegraphics[width=\linewidth]{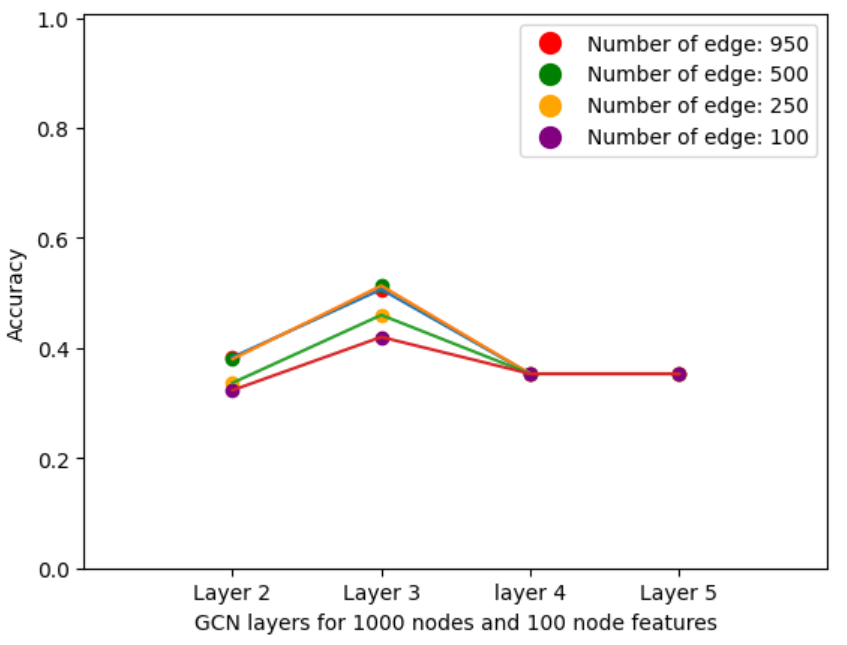}
    \caption*{(d)}
\end{minipage}
\caption{Accuracy of GCN layers for four graph datasets in various conditions.}
\label{fig:gcn_accuracy}
\end{figure*}

\Cref{fig:gcn_accuracy}(a) demonstrates the node classification accuracy of Graph Convolutional Networks (GCNs) across graphs with varying edge densities. Specifically, four graphs, each consisting of 500 nodes and 100-dimensional node features but differing in the number of edges, are analyzed. The horizontal axis represents the number of GCN layers, and the vertical axis indicates the classification accuracy. In particular, regardless of edge count, three-layer GCNs consistently achieve the highest accuracy. In \Cref{fig:gcn_accuracy}(a) the graphs are arranged along the horizontal axis in descending order according to their edge counts illustrating how increased connectivity impacts these metrics. Fiedler values, again ordered by edge count.

\begin{figure}[t]
\centering
\begin{minipage}[b]{0.45\linewidth}
    \centering
    \includegraphics[width=\linewidth]{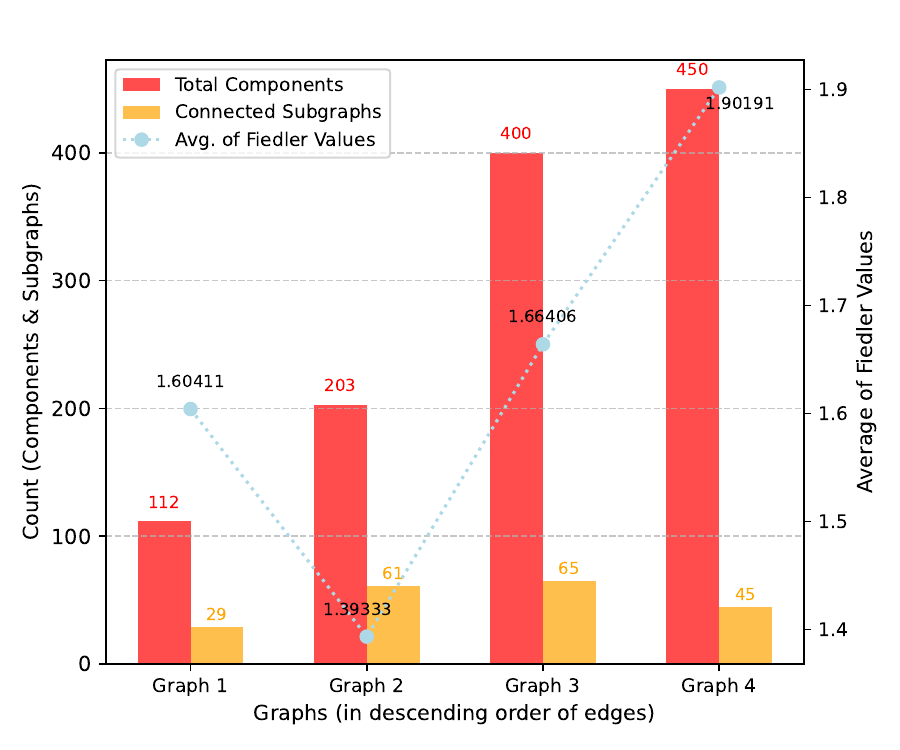}
    \caption*{(a)}
\end{minipage}
\hfill
\begin{minipage}[b]{0.45\linewidth}
    \centering
    \includegraphics[width=\linewidth]{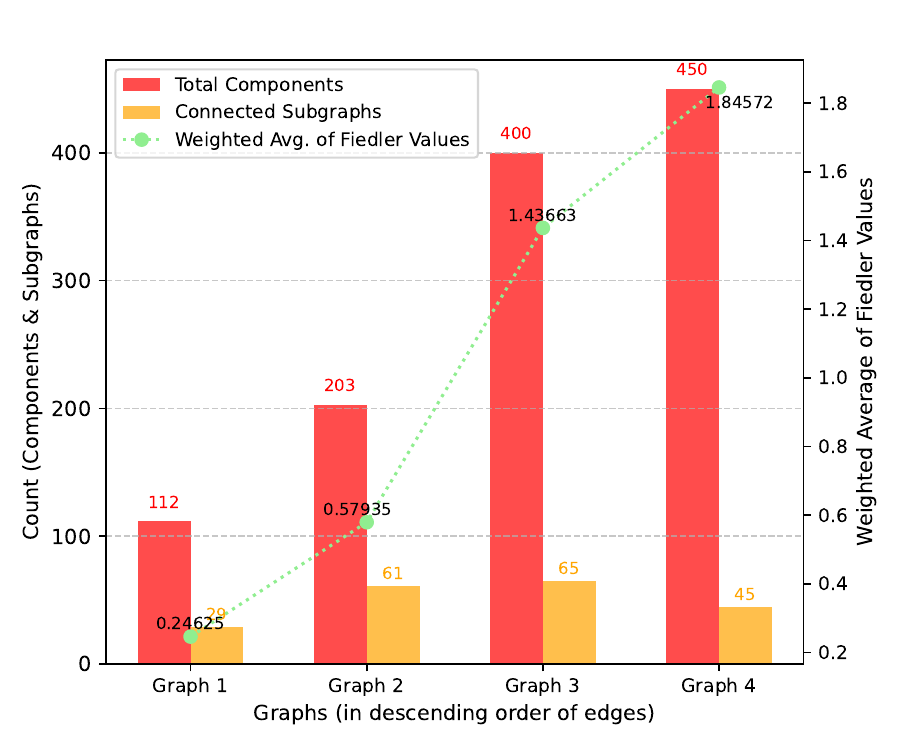}
    \caption*{(b)}
\end{minipage}
\caption{Structural metrics derived from the four 500‑node graphs used in Fig.\,\ref{fig:gcn_accuracy}(a). (a) Average of Fiedler values (b) Weighted average of Fiedler values.}
\label{fig:fiedler}
\end{figure}

In \Cref{fig:fiedler}(a), a comparison is depicted among the graphs of \Cref{fig:gcn_accuracy}(a) based on the average of the Fiedler values, total components of the graph, and total number of connected subgraphs. Similarly, \Cref{fig:fiedler}(b) further refines this analysis by considering the weighted average of Fiedler values, from \Cref{fig:gcn_accuracy}(a)  in descending order based on the number of edges in the graph.
The weighted average calculation involves multiplying each component's Fiedler value by its corresponding number of nodes, summing these products, and then dividing by the total number of nodes. This metric captures a nuanced perspective of the graph's spectral connectivity relative to its structural complexity.

\begin{figure}[t]
\centering
\begin{minipage}[b]{0.45\linewidth}
    \centering
    \includegraphics[width=\linewidth]{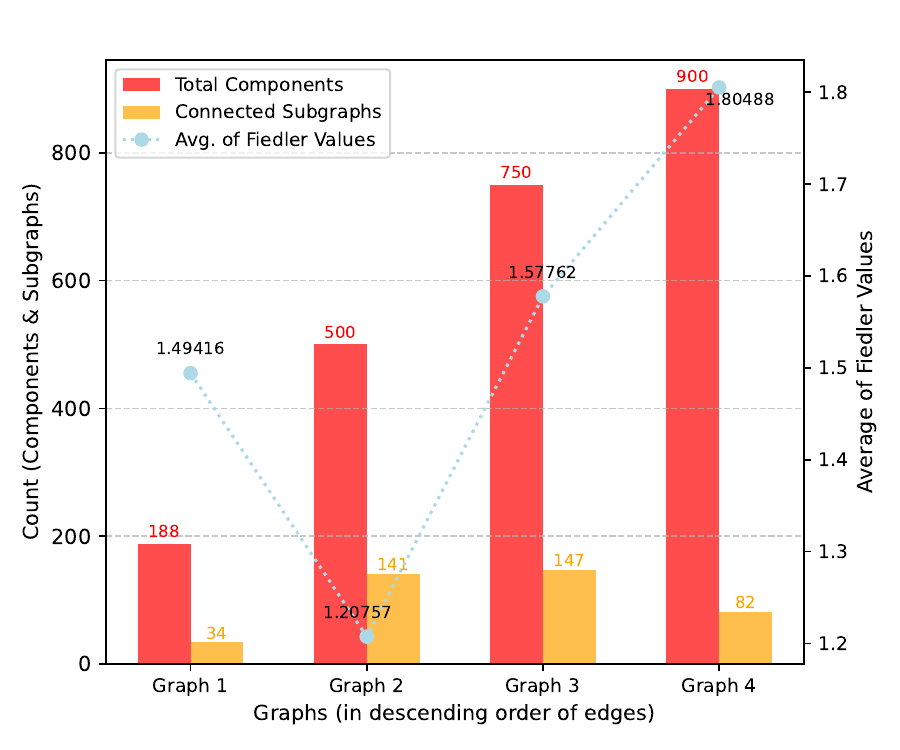}
    \caption*{(a)}
\end{minipage}
\hfill
\begin{minipage}[b]{0.45\linewidth}
    \centering
    \includegraphics[width=\linewidth]{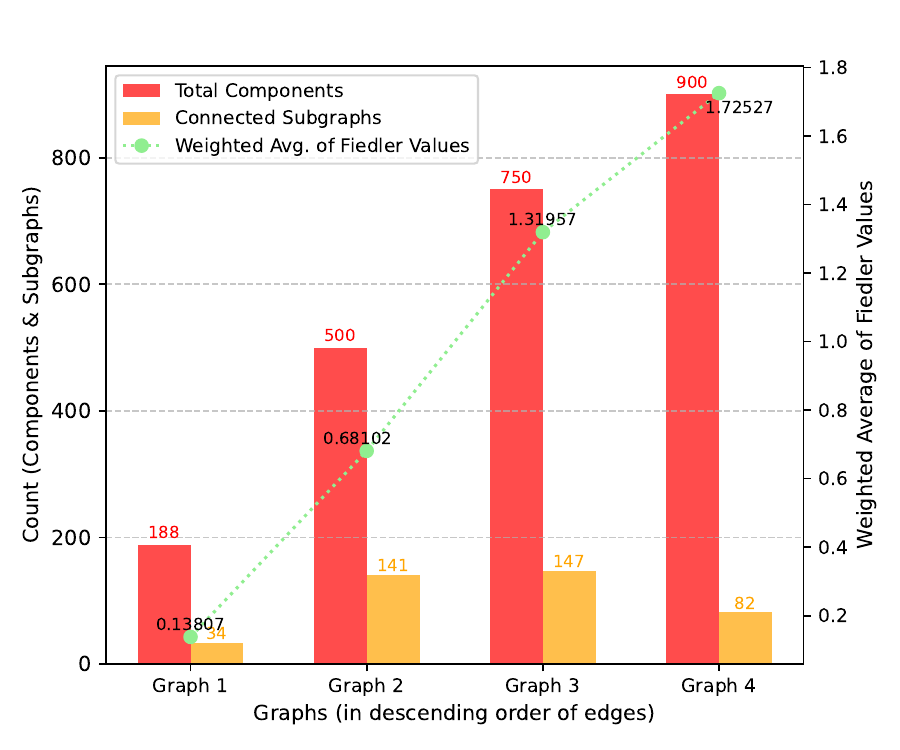}
    \caption*{(b)}
\end{minipage}
\caption{Structural metrics derived from the four 1000‑node graphs used in Fig.\,\ref{fig:gcn_accuracy}(b). (a) Average of Fiedler values (b) Weighted average of Fiedler values.}
\label{fig4}
\end{figure}

\Cref{fig4}(a) depicts the total number of components, total number of connected subgraphs, and an average of Fiedler values of \Cref{fig:gcn_accuracy}(b). The data of the graphs are also presented in descending order based on the number of edges, highlighting how these structural attributes vary systematically with graph connectivity. In the same way, \Cref{fig4}(b) further elaborates by displaying the weighted average Fiedler values of \Cref{fig:gcn_accuracy}(b). This metric emphasizes how spectral connectivity scales with the structural composition of the graph.

\begin{figure}[t]
\centering
\begin{minipage}[b]{0.45\linewidth}
    \centering
    \includegraphics[width=\linewidth]{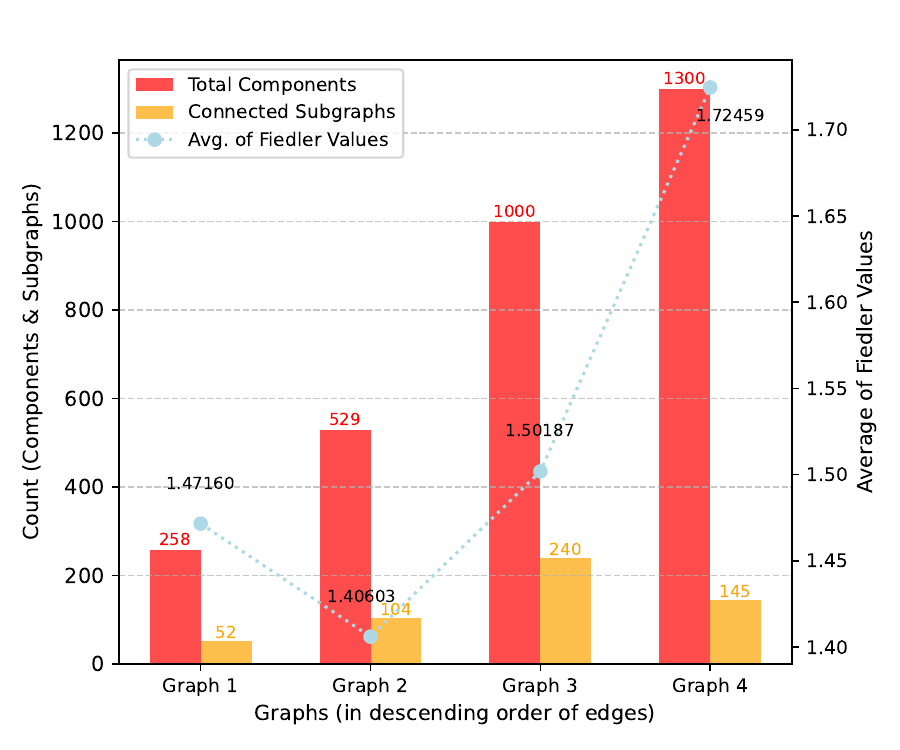}
    \caption*{(a)}
\end{minipage}
\hfill
\begin{minipage}[b]{0.45\linewidth}
    \centering
    \includegraphics[width=\linewidth]{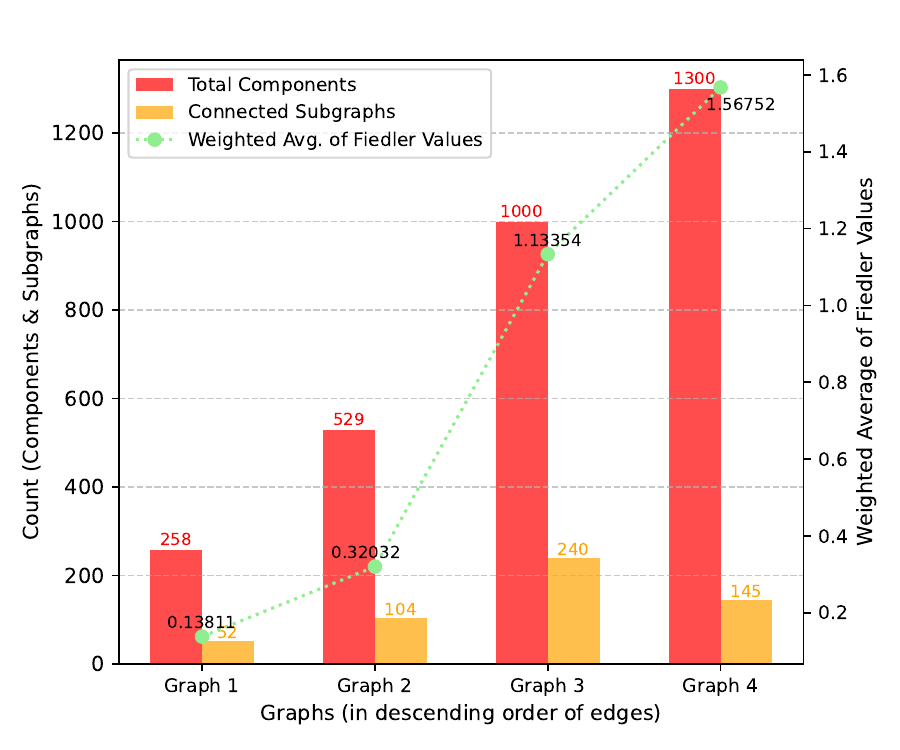}
    \caption*{(b)}
\end{minipage}
\caption{Structural metrics derived from the four 1500‑node graphs used in \Cref{fig:gcn_accuracy}(c). (a) Average of Fiedler values (b) Weighted average of Fiedler values.}
\label{fig6}
\end{figure}

\Cref{fig:gcn_accuracy}(c) illustrates the accuracy of GCN for four graph datasets of 1500 nodes and 100 node features. Classification accuracy decreases at the extremes of edge connectivity, being lower for both densely and sparsely connected graphs.

\Cref{fig6}(a) illustrates the average of Fiedler values, number of total components, and connected subgraphs of \Cref{fig:gcn_accuracy}(c). A comparison is made among the four graphs in descending order based on the number of edges. Similarly, \Cref{fig6}(b) represents the weighted average of the Fiedler values of \Cref{fig:gcn_accuracy}(c).
\begin{figure}[t]
\centering
\begin{minipage}[b]{0.45\linewidth}
    \centering
    \includegraphics[width=\linewidth]{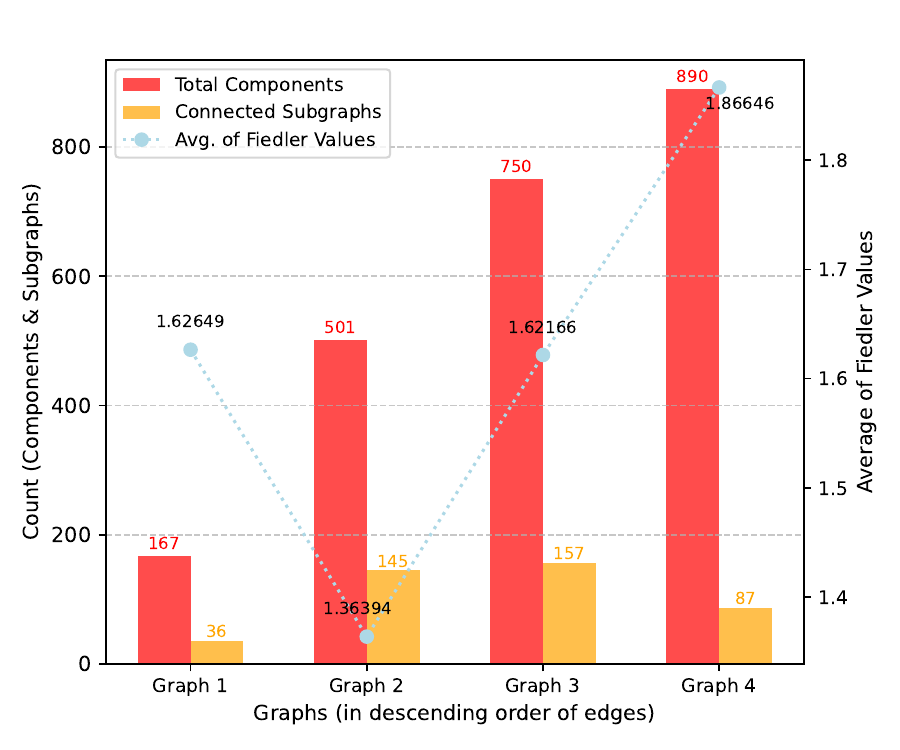}
    \caption*{(a)}
\end{minipage}
\hfill
\begin{minipage}[b]{0.45\linewidth}
    \centering
    \includegraphics[width=\linewidth]{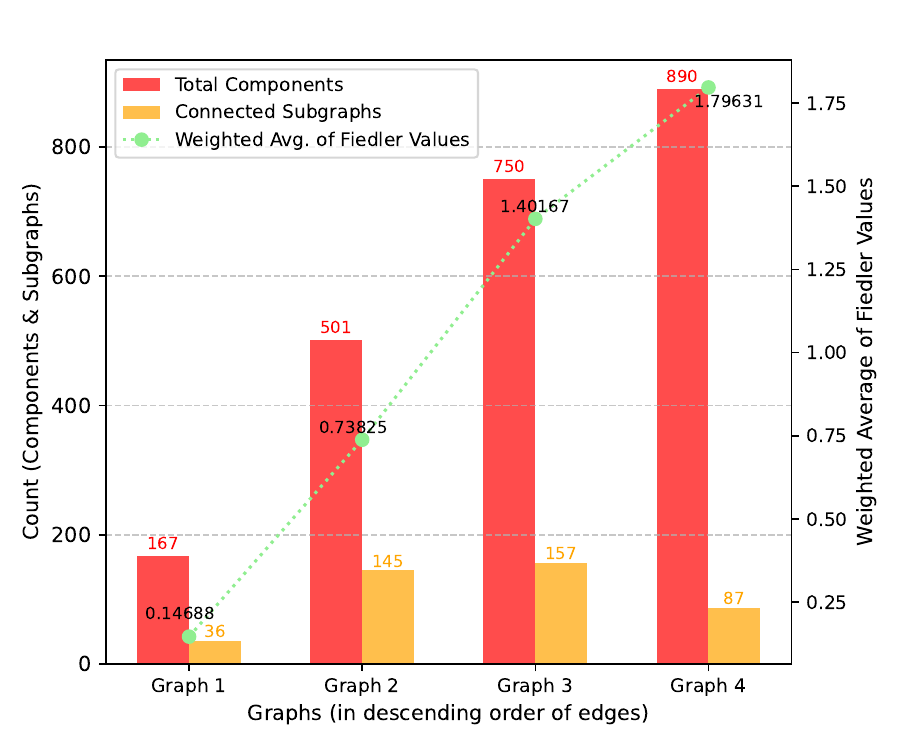}
    \caption*{(b)}
\end{minipage}
\caption{Structural metrics derived from the four 1000‑node graphs used in Fig.\,\ref{fig:gcn_accuracy}(d). (a) Average of Fiedler values (b) Weighted average of Fiedler values.}
\label{fig8}
\end{figure}

\Cref{fig:gcn_accuracy}(d) presents node classification accuracy across multiple GCN layers on a single initial graph comprising 1000 nodes, 100-dimensional node features, and initially 950 edges. To investigate how connectivity influences accuracy and spectral properties, edges were gradually and randomly removed (from 950 down to 500, 250, and 100). The horizontal axis denotes the number of GCN layers, while the vertical axis indicates accuracy. Consistent with previous observations, accuracy peaks at three layers regardless of the graph's connectivity, highlighting an optimal GCN depth under these conditions. The data are organized in descending order by edge count, clearly showing how connectivity reduction affects these structural measures.

\Cref{fig8}(a) and \Cref{fig8}(b) depict the average of Fiedler values and also the weighted average of the Fiedler values of \Cref{fig:gcn_accuracy}(d). Comparison is made in descending order based on the number of edges. This weighted metric further clarifies how decreasing graph connectivity directly affects spectral characteristics critical to GCN performance.
Overall, these figures collectively emphasize how graph structural properties, particularly as captured by Fiedler values and connectivity metrics, correlate with GCN performance and optimal model depth.

Table \ref{table_real_data} depicts the performance of GCN in Cora CiteSeer and Polblogs datasets. The accuracy of applying GCN is dropped immediately after applying 3 convolution layers for Cora and CiteSeer dataset and for Polblogs dataset it takes 5 layers. We have preprocessed the PolBlogs dataset by removing communities that have only one member to successfully perform the node classification task. After preprocessing, the number of nodes is 1224 and the classes are 11. Since the dataset does not include node features, we assign each node a unique one-hot encoded vector as its feature representation, resulting in an identity matrix of size \(N \times N \), where \(N\) is the number of nodes in the preprocessed graph.

Table \ref{sample-table3} illustrates the average Fiedler values and weighted average Fiedler values of Cora, CiteSeer and Polblogs datasets.

\begin{table*}[t]
\centering
\begin{tabular}{lrrrrrrr}
\toprule
\textbf{Dataset} & \textbf{Nodes} & \textbf{Feat.} & \textbf{Edges} & \textbf{Acc@2} & \textbf{Acc@3} & \textbf{Acc@4} & \textbf{Acc@5} \\
\midrule
Cora     & 2,708 & 1,433 & 10,556 & 0.8150 & 0.7910 & 0.7410     & 0.6730     \\
CiteSeer & 3,327 & 3,703 & 9,104  & 0.7120 & 0.6640 & --     & 0.6030    \\
Polblogs & 1,490 & 1,224 & 16,718 & 0.9565 & 0.9592 & 0.9620 & 0.9484 \\
\bottomrule
\end{tabular}
\caption{Accuracy of GCN with varying convolution layers}
\label{table_real_data}
\end{table*}

\begin{table}[t]
\centering
\begin{tabular}{lcc}
\toprule
\textbf{Dataset} & \textbf{Avg. of FV} & \textbf{Weighted avg. of FV} \\
\midrule
Cora     & 1.62833 & 0.1024  \\
CiteSeer & 1.54106 & 0.4205  \\
Polblogs & 1.50001 & 0.04615 \\
\bottomrule
\end{tabular}
\caption{Average of Fiedler values and weighted average of Fiedler values of real-world datasets}
\label{sample-table3}
\end{table}

Our experiments reveal an optimal range of Fiedler values associated with strong GCN performance. Specifically, we observe diminished classification performance when algebraic connectivity is either too low or too high, even with identical numbers of nodes and features. Empirically, average Fiedler values exceeding approximately 1.4 consistently correlate with reduced performance. In contrast, we find the strongest performance when the weighted average of Fiedler values falls between approximately 0.5 and 1.2. This trend holds consistently across multiple experiments, indicating suboptimal performance for both excessively dense (high edge count) and excessively sparse (low edge count) graph structures. This enables the use of $\lambda2$ as a pre-training diagnostic to inform architecture choices and potentially reduce expensive hyperparameter search
\section{Conclusion}

Graph Convolutional Networks (GCNs) face well-documented performance limitations, particularly in relation to the structural properties of their input graphs. This work highlights the Fiedler value as a theoretically grounded and empirically informative metric for assessing and predicting GCN performance. By quantifying algebraic connectivity, the Fiedler value offers practical guidance on whether a given graph falls within a range where shallow or deep GCNs—or GCNs at all—are likely to be effective. Our findings suggest that incorporating spectral characteristics, especially the Fiedler value, into model selection and experimental design can improve performance across diverse datasets.

The broader impacts of this work include reducing the time to create effective graph algorithms by choosing an appropriate number of convolution layers quickly, resulting in fewer training runs and a smaller carbon footprint.  This study is limited to node classification tasks on graphs with multiple components. Future work will explore broader empirical validation, including alternative tasks and metrics such as feature distance scores, to further elucidate the role of spectral properties in GCN performance.

\bibliography{aaai2026}

\begin{thebibliography}{36}
\providecommand{\natexlab}[1]{#1}

\bibitem[{Adamic and Glance(2005)}]{adamic2005political}
Adamic, L.~A.; and Glance, N. 2005.
\newblock The political blogosphere and the 2004 US election: divided they
  blog.
\newblock In \emph{Proceedings of the 3rd International Workshop on Link
  Discovery}, 36--43.

\bibitem[{Alon and Yahav(2020)}]{alon2020bottleneck}
Alon, U.; and Yahav, E. 2020.
\newblock On the bottleneck of graph neural networks and its practical
  implications, arXiv.
\newblock \emph{arXiv preprint arXiv:2006.05205}.

\bibitem[{Barriere et~al.(2013)Barriere, Huemer, Mitsche, and
  Orden}]{barriere2013fiedler}
Barriere, L.; Huemer, C.; Mitsche, D.; and Orden, D. 2013.
\newblock On the Fiedler value of large planar graphs.
\newblock \emph{Linear Algebra and its Applications}, 439(7): 2070--2084.

\bibitem[{Chen et~al.(2020)Chen, Lin, Li, Li, Zhou, and
  Sun}]{chen2020measuring}
Chen, D.; Lin, Y.; Li, W.; Li, P.; Zhou, J.; and Sun, X. 2020.
\newblock Measuring and relieving the over-smoothing problem for graph neural
  networks from the topological view.
\newblock In \emph{Proceedings of the AAAI conference on artificial
  intelligence}, volume~34, 3438--3445.

\bibitem[{Chung(1997)}]{Chung1997}
Chung, F. R.~K. 1997.
\newblock \emph{Spectral Graph Theory}, volume~92 of \emph{CBMS Regional
  Conference Series in Mathematics}.
\newblock American Mathematical Society.

\bibitem[{Daigavane, Ravindran, and
  Aggarwal(2021)}]{daigavane2021understanding}
Daigavane, A.; Ravindran, B.; and Aggarwal, G. 2021.
\newblock Understanding convolutions on graphs.
\newblock \emph{Distill}, 6(9): e32.

\bibitem[{Defferrard, Bresson, and
  Vandergheynst(2016)}]{defferrard2016convolutional}
Defferrard, M.; Bresson, X.; and Vandergheynst, P. 2016.
\newblock Convolutional neural networks on graphs with fast localized spectral
  filtering.
\newblock \emph{Advances in neural information processing systems}, 29.

\bibitem[{Elinas and Bonilla(2022)}]{elinas2022addressing}
Elinas, P.; and Bonilla, E.~V. 2022.
\newblock Addressing Over-Smoothing in Graph Neural Networks via Deep
  Supervision.
\newblock \emph{arXiv preprint arXiv:2202.12508}.

\bibitem[{Fesser and Weber(2025)}]{fesser2025performance}
Fesser, L.; and Weber, M. 2025.
\newblock Performance Heterogeneity in Graph Neural Networks: Lessons for
  Architecture Design and Preprocessing.
\newblock \emph{arXiv preprint arXiv:2503.00547}.

\bibitem[{Fiedler(1989)}]{fiedler1989laplacian}
Fiedler, M. 1989.
\newblock Laplacian of graphs and algebraic connectivity.
\newblock \emph{Banach Center Publications}, 1(25): 57--70.

\bibitem[{Gama et~al.(2020)Gama, Isufi, Leus, and Ribeiro}]{gama2020graphs}
Gama, F.; Isufi, E.; Leus, G.; and Ribeiro, A. 2020.
\newblock Graphs, convolutions, and neural networks: From graph filters to
  graph neural networks.
\newblock \emph{IEEE Signal Processing Magazine}, 37(6): 128--138.

\bibitem[{Giovanni et~al.(2023)Giovanni, Rowbottom, Chamberlain, Markovich, and
  Bronstein}]{digio2023energies}
Giovanni, F.~D.; Rowbottom, J.; Chamberlain, B.~P.; Markovich, T.; and
  Bronstein, M.~M. 2023.
\newblock Understanding Convolution on Graphs via Energies.
\newblock \emph{Transactions on Machine Learning Research}.
\newblock Section 2.2 formally defines the Dirichlet energy of a graph signal.

\bibitem[{Giraldo, Malliaros, and Bouwmans(2022)}]{giraldo2022understanding}
Giraldo, J.~H.; Malliaros, F.~D.; and Bouwmans, T. 2022.
\newblock Understanding the relationship between over-smoothing and
  over-squashing in graph neural networks.
\newblock \emph{arXiv preprint arXiv:2212.02374}.

\bibitem[{Grone, Merris, and Sunder(1990)}]{grone1990laplacian}
Grone, R.; Merris, R.; and Sunder, V. 1990.
\newblock The Laplacian spectrum of a graph.
\newblock \emph{SIAM Journal on matrix analysis and applications}, 11(2):
  218--238.

\bibitem[{He, Wei, and Wen(2022)}]{he2022convolutional}
He, M.; Wei, Z.; and Wen, J.-R. 2022.
\newblock Convolutional neural networks on graphs with chebyshev approximation,
  revisited.
\newblock \emph{Advances in Neural Information Processing Systems}, 35:
  7264--7276.

\bibitem[{Keriven(2022)}]{keriven2022not}
Keriven, N. 2022.
\newblock Not too little, not too much: a theoretical analysis of graph (over)
  smoothing.
\newblock \emph{Advances in Neural Information Processing Systems}, 35:
  2268--2281.

\bibitem[{Kipf and Welling(2016)}]{kipf2016semi}
Kipf, T.~N.; and Welling, M. 2016.
\newblock Semi-supervised classification with graph convolutional networks.
\newblock \emph{arXiv preprint arXiv:1609.02907}.

\bibitem[{Kohayakawa, R{\"o}dl, and Schacht(2003)}]{kohayakawa2003discrepancy}
Kohayakawa, Y.; R{\"o}dl, V.; and Schacht, M. 2003.
\newblock Discrepancy and eigenvalues of cayley graphs.
\newblock \emph{Eurocomb 2003}, 145.

\bibitem[{Li et~al.(2019)Li, Muller, Thabet, and Ghanem}]{li2019deepgcns}
Li, G.; Muller, M.; Thabet, A.; and Ghanem, B. 2019.
\newblock Deepgcns: Can gcns go as deep as cnns?
\newblock In \emph{Proceedings of the IEEE/CVF international conference on
  computer vision}, 9267--9276.

\bibitem[{Li, Han, and Wu(2018)}]{li2018deeper}
Li, Q.; Han, Z.; and Wu, X.-M. 2018.
\newblock Deeper insights into graph convolutional networks for semi-supervised
  learning.
\newblock In \emph{Proceedings of the AAAI conference on artificial
  intelligence}, volume~32.

\bibitem[{Marsden(2013)}]{marsden2013eigenvalues}
Marsden, A. 2013.
\newblock Eigenvalues of the laplacian and their relationship to the
  connectedness of a graph.
\newblock \emph{University of Chicago, REU}.

\bibitem[{Mirchev(2016)}]{mirchev2016beyond}
Mirchev, M.~J. 2016.
\newblock Beyond algebraic connectivity of graphs-evaluation of topology, based
  on spectral clustering.
\newblock \emph{Electrotechnica \& Electronica (E+ E)}, 51.

\bibitem[{Mohar(1991)}]{mohar1991eigenvalues}
Mohar, B. 1991.
\newblock Eigenvalues, diameter, and mean distance in graphs.
\newblock \emph{Graphs and combinatorics}, 7(1): 53--64.

\bibitem[{Mostafa and Nassar(2020)}]{mostafa2020permutohedral}
Mostafa, H.; and Nassar, M. 2020.
\newblock Permutohedral-gcn: Graph convolutional networks with global
  attention.
\newblock \emph{arXiv preprint arXiv:2003.00635}.

\bibitem[{Oono and Suzuki(2019)}]{oono2019graph}
Oono, K.; and Suzuki, T. 2019.
\newblock Graph neural networks exponentially lose expressive power for node
  classification.
\newblock \emph{arXiv preprint arXiv:1905.10947}.

\bibitem[{Rossi and Ahmed(2015)}]{rossi2015network}
Rossi, R.~A.; and Ahmed, N.~K. 2015.
\newblock The Network Data Repository with Interactive Graph Analytics and
  Visualization.
\newblock In \emph{Proceedings of the 29th AAAI Conference on Artificial
  Intelligence}.

\bibitem[{Rusch, Bronstein, and Mishra(2023)}]{rusch2023survey}
Rusch, T.~K.; Bronstein, M.~M.; and Mishra, S. 2023.
\newblock A survey on oversmoothing in graph neural networks.
\newblock \emph{arXiv preprint arXiv:2303.10993}.

\bibitem[{Shuman et~al.(2013)Shuman, Narang, Frossard, Ortega, and
  Vandergheynst}]{shuman2013emerging}
Shuman, D.~I.; Narang, S.~K.; Frossard, P.; Ortega, A.; and Vandergheynst, P.
  2013.
\newblock The emerging field of signal processing on graphs: Extending
  high-dimensional data analysis to networks and other irregular domains.
\newblock \emph{IEEE Signal Processing Magazine}, 30(3): 83--98.

\bibitem[{Sun(2022)}]{sun2022over}
Sun, F. 2022.
\newblock Over-smoothing effect of graph convolutional networks.
\newblock \emph{arXiv preprint arXiv:2201.12830}.

\bibitem[{Veli{\v{c}}kovi{\'c}(2023)}]{velivckovic2023everything}
Veli{\v{c}}kovi{\'c}, P. 2023.
\newblock Everything is connected: Graph neural networks.
\newblock \emph{Current Opinion in Structural Biology}, 79: 102538.

\bibitem[{Von~Luxburg(2007)}]{von2007tutorial}
Von~Luxburg, U. 2007.
\newblock A tutorial on spectral clustering.
\newblock \emph{Statistics and computing}, 17: 395--416.

\bibitem[{Wu et~al.(2022)Wu, Liang, Zheng, Guo, and Tang}]{wu2022improving}
Wu, B.; Liang, X.; Zheng, X.; Guo, Y.; and Tang, H. 2022.
\newblock Improving Dynamic Graph Convolutional Network with Fine-Grained
  Attention Mechanism.
\newblock In \emph{ICASSP 2022-2022 IEEE International Conference on Acoustics,
  Speech and Signal Processing (ICASSP)}, 3938--3942. IEEE.

\bibitem[{Xu et~al.(2023)Xu, Zhao, Wei, and Li}]{xu2023comprehensive}
Xu, X.; Zhao, X.; Wei, M.; and Li, Z. 2023.
\newblock A comprehensive review of graph convolutional networks: approaches
  and applications.
\newblock \emph{Electronic Research Archive}, 31(7): 4185--4215.

\bibitem[{Yang, Cohen, and Salakhudinov(2016)}]{yang2016revisiting}
Yang, Z.; Cohen, W.; and Salakhudinov, R. 2016.
\newblock Revisiting semi-supervised learning with graph embeddings.
\newblock In \emph{International conference on machine learning}, 40--48. PMLR.

\bibitem[{Zhang et~al.(2019)Zhang, Tong, Xu, and Maciejewski}]{zhang2019graph}
Zhang, S.; Tong, H.; Xu, J.; and Maciejewski, R. 2019.
\newblock Graph convolutional networks: a comprehensive review.
\newblock \emph{Computational Social Networks}, 6(1): 1--23.

\bibitem[{Zhou et~al.(2020)Zhou, Cui, Hu, Zhang, Yang, Liu, Wang, Li, and
  Sun}]{zhou2020graph}
Zhou, J.; Cui, G.; Hu, S.; Zhang, Z.; Yang, C.; Liu, Z.; Wang, L.; Li, C.; and
  Sun, M. 2020.
\newblock Graph neural networks: A review of methods and applications.
\newblock \emph{AI open}, 1: 57--81.

\end{thebibliography}


\appendix

\section*{A. Appendix}

\subsection{A.1 Notation}
Notations are depicted in Table \ref{tab:notation}
\begin{table*}[t]
\centering
\small  
\begin{tabular}{@{}lll@{}}
\toprule
\textbf{Symbol} & \textbf{Meaning} & \textbf{Shape} \\
\midrule
\(\mathcal V\) & vertex set & – \\
\(n:=|\mathcal V|\) & number of vertices & – \\
\(\mathbf A\) & adjacency matrix & \(n\times n\) \\
\(\mathbf D\) & degree matrix & \(n\times n\) \\
\(\mathbf L=\mathbf D-\mathbf A\) & combinatorial Laplacian & \(n\times n\) \\
\(\mathbf X=[x_1^\top\,\dots\,x_n^\top]^\top\) & input feature matrix & \(n\times m\) \\
\(V^{(k)}\) & hidden feature matrix after \(k\)-th GCN layer & \(n\times m_k\) \\
\(u_i,\lambda_i\) & \(i\)-th eigenpair of \(\mathbf L\) & \(n\), scalar \\
\(w_i\) & spectral coefficient vectors & \(m\) \\
\(\delta_{ij}\) & Kronecker delta (\(=1\) if \(i=j\)) & – \\
\(\rho_k\) & normalised feature–distance score (Eq.\,\ref{eq:rho_k}) & scalar \\
\bottomrule
\end{tabular}
\caption{Glossary of symbols used throughout the paper.}
\label{tab:notation}
\end{table*}

\subsection{A.2 Layer-wise Propagation rule of GCN}
\begin{equation}
H^{(l+1)} = \sigma \left( \tilde{D}^{-\frac{1}{2}} \tilde{A} \tilde{D}^{-\frac{1}{2}} H^{(l)} W^{(l)} \right)
\end{equation} 
where:
\begin{itemize}
  \item $H^{(l)}$ is the feature matrix of layer $l$
  \item $W^{(l)}$ is the learnable weight matrix for layer $l$ 
  \item $\tilde{A} = A + I_N$ is the adjacency matrix with added self-loop
  \item $\tilde{D}$ is the degree matrix corresponding to $\tilde{A}$
  \item $\sigma$ is a non-linear activation function
\end{itemize}
$H^{(0)}$ is the input feature matrix, and the output of the final layer $H^{(L)}$ (after $L$ layers) is used for tasks like classification.
In the above formula, $ \tilde{D}^{-\frac{1}{2}} \tilde{A} \tilde{D}^{-\frac{1}{2}}$ is derived from the normalized graph Laplacian \cite{kipf2016semi}.
\subsection{A.3 Predicting Expander-like Graph}
A larger Fiedler value indicates an expander-like graph \cite{kohayakawa2003discrepancy}, and GCNs are known to suffer from information loss in such  graphs \cite{oono2019graph}.
For an arbitrary vertex \( \omega \in V(G) \) and a non-negative whole number $k$, let \( R_k = R_k(\omega) \) be the set of vertices of \( G \) which are at distance at most \( l \) from \( \omega \) and \( r_k = |R_k| \). $r_k$ has exponential growth, i.e. \( r_k \geq \alpha^k \) where \( \alpha \) is a constant bounded below as a function of the Fiedler value (\( \lambda_2 \)) \cite{mohar1991eigenvalues}.  Larger Fiedler values indicate that exponentially more information reaches nodes with additional rounds of convolution, increasing the chances of over-squashing.

\subsection{A.4 Under-Reaching and Over-Squashing}
The diameter of a graph is the longest shortest path between any two nodes in the graph, and it has a profound
effect on GCN performance. Computing the exact diameter of a large graph is computationally intensive and often impractical, making diameter approximation a more feasible approach.
The Fiedler value can be used to estimate the diameter of a real‑world graph. This assessment may assist in mitigating issues related to under‑reaching or over‑squashing. When a node does not receive needed information about other nodes that are far away, this is called \emph{under‑reaching} \cite{alon2020bottleneck}. Recall that if we apply $l$ convolution layers in a GCN, nodes that are separated by a shortest path of more than $l$ hops will not be ``aware'' of each other. On the other hand, GCNs suffer from \emph{over‑squashing} when exponentially growing information about increasingly distal parts of the graph is forced into a fixed‑size vector. GCN models applied to graphs with large diameters often face over‑squashing \cite{alon2020bottleneck}. In both cases, knowing the diameter of a graph is useful to determine whether under‑reaching or over‑squashing might be a problem. Again, the Fiedler value can be used to bound the diameter of a graph $G$ of order $n$\cite{mohar1991eigenvalues}.

The \textbf{lower bound} on the diameter is:
\[
\operatorname{diam}(G) \;\ge\; \frac{4}{n\lambda_2}.
\]
and the \textbf{upper bound} is:
\[
\operatorname{diam}(G) \;\le\;  
2 \Bigl\lceil \sqrt{\frac{2\Delta}{\lambda_2}} \,\log_2 n \Bigr\rceil .
\]
Here $\Delta$ is the maximum node degree of graph $G$.

\subsection{A.5 Proof of Lemma 1}
\begin{proof}
Using the identity \( L = D - A \) and the cyclic property of the trace 
\[
\operatorname{Tr}(V^{\top} L V) = \operatorname{Tr}(V^{\top} D V) - \operatorname{Tr}(V^{\top} A V).
\]
Since the graph is undirected, each unordered edge \( \{i,j\} \in E \) appears exactly once, and \( D_{ii} = d_i = \sum_j A_{ij} \). Therefore,
\[
\operatorname{Tr}(V^{\top} D V) = \sum_{i=1}^{n} d_i \|v_i\|_2^2 = \sum_{\{i,j\} \in E} \left( \|v_i\|_2^2 + \|v_j\|_2^2 \right),
\]
and
\[
\operatorname{Tr}(V^{\top} A V) = \sum_{i,j} A_{ij} v_i^{\top} v_j = 2 \sum_{\{i,j\} \in E} v_i^{\top} v_j.
\]
Subtracting the two expressions gives:
\[
\operatorname{Tr}(V^{\top} L V) =
\sum_{\{i,j\} \in E} \left( \|v_i\|_2^2 + \|v_j\|_2^2 - 2 v_i^{\top} v_j \right)
\\
= \sum_{\{i,j\} \in E} \|v_i - v_j\|_2^2,
\]
which completes the proof.
\end{proof}

\subsection{A.6 Proof of Lemma 2}
\begin{proof}
By definition, \( L = D - A \), where \( D \) is the degree matrix and \( A \) is the adjacency matrix. Since \( G \) is undirected, each row of \( L \) sums to zero. That is, for all \( i \in \{1, \dots, n\} \), we have \( \sum_{j=1}^{n} L_{ij} = 0 \). It follows that \( L \mathbf{1} = 0 \), so \( \lambda_1 = 0 \) is an eigenvalue of \( L \), with eigenvector \( \mathbf{1} \). Normalizing \( \mathbf{1} \) yields the unit eigenvector \( u_1 = \tfrac{1}{\sqrt{n}} \mathbf{1} \).
\end{proof}
\subsection{A.7 Proof of Lemma 3}

\begin{proof}
The first Laplacian eigenvector is the normalized all-ones vector
\(u_{1}=\tfrac1{\sqrt n}\mathbf 1\) (see Lemma~\ref{lem:first-eigenpair}).
Left-multiplying \((\ast)\) by \(\mathbf 1^{\top}\) gives
\[
  \mathbf 1^{\top}\bar V
  \;=\;
  \sum_{i=1}^{n}\mathbf 1^{\top}u_i\,w_i^{\!\top}
  \;=\;
  \sqrt n\,w_{1}^{\!\top},
\]
since \(\mathbf 1^{\top}u_i = 0\) for \(i\ge2\) by orthogonality.
The centering condition implies \(\mathbf 1^{\top}\bar V = 0\), so
\(w_{1}^{\!\top}=0\) and thus 
\(w_{1}=0_m\).
\end{proof}

\subsection{ A.8 Proof of Lemma 4}
\begin{proof}
Centering implies \(\mathbf 1^{\top}\bar V=0_{1\times m}\); since
\(\mathbf 1^{\top}\bar u_i=0\) for \(i\ge2\) and
\(\mathbf 1^{\top}\bar u_1=\sqrt n\),
we must have \(w_1=0_m\).
Using the orthogonality \(\bar u_i^{\top}\bar u_j=\delta_{ij}\),
\begin{eqnarray*}
  \operatorname{Tr}(\bar V^{\top}\bar V)
  & = & 
  \operatorname{Tr}\!\Bigl(
    \bigl[\sum_{i=2}^{n} w_i\,\bar u_i^{\top}\bigr]
    \bigl[\sum_{j=2}^{n} \bar u_j\,w_j^{\top}\bigr]
  \Bigr)
  \\
  & = &
  \sum_{i=2}^{n}\operatorname{Tr}(w_i w_i^{\top})
  \\
  & = &
  \sum_{i=2}^{n}\|w_i\|_2^{2}.
\end{eqnarray*}

But \(\operatorname{Tr}(\bar V^{\top}\bar V)
      =\sum_{i=1}^{n}\|v_i\|_2^{2}\), completing the proof.
\end{proof}

\subsection{A.9 Proof of Lemma 5}

\begin{proof}
Centering forces \(w_1=0_m\) exactly as in
Lemma \ref{lem:no-constant-mode}; thus
\(\bar V=\sum_{i=2}^{n} u_i w_i^{\!\top}\).
Using cyclicity of the trace,
\begin{eqnarray*}
  \operatorname{Tr}(\bar V^{\top}L\bar V)
  &=&\operatorname{Tr}\!\Bigl(
      \Bigl[\sum_{i=2}^{n} w_i\,u_i^{\top}\Bigr]
      L
      \Bigl[\sum_{j=2}^{n} u_j\,w_j^{\!\top}\Bigr]
    \Bigr)
    \\
  &=&\sum_{i=2}^{n}\sum_{j=2}^{n}
     w_i^{\top}\underbrace{u_i^{\top}L u_j}_{=\;\lambda_j\delta_{ij}}w_j \\
  &=&\sum_{i=2}^{n}\lambda_i\,w_i^{\top}w_i
  \\
  &=&\sum_{i=2}^{n}\lambda_i\,\|w_i\|_2^{2}.
\end{eqnarray*}
The first equality uses the identity
\(\sum_{(i,j)\in E}\|v_i-v_j\|_2^{2}
     =\operatorname{Tr}(\bar V^{\top}L\bar V)\)
from Lemma \ref{lem:dirichlet}.
\end{proof}

\subsection{A.10 Proof of Theorem 1}
\begin{proof}[Proof of Theorem~\ref{thm:fiedler-min}]
\textbf{Step 1 : Rayleigh--Ritz expression for the objective.}
By Lemma~\ref{lem:vectorcase} the total feature distance can be written
as the quadratic form
\[
  \mathcal E(\bar V)
  \;=\;
  \sum_{(i,j)\in E}\!\|v_i-v_j\|_{2}^{2}
  \;=\;
  \sum_{i=2}^{n}\lambda_i\,\|w_i\|_2^{2},
  \tag{A}
\]
where the coefficient vectors \(w_i\in\mathbb{R}^{m}\) come from the
outer‑product expansion
\(\bar V=\sum_{i\ge2}u_i w_i^{\!\top}\).

\smallskip
\textbf{Step2: Constraint in the coefficient space.}
Lemma~\ref{lem:frobenius} (with \(w_1=0_m\), enforced by the centering
assumption) converts the Frobenius‑norm constraint into
\[
  \sum_{i=2}^{n}\|w_i\|_2^{2}
  \;=\;
  n.
  \tag{B}
\]

\smallskip
\textbf{Step3: Optimisation.}
Because the eigenvalues are ordered
\(0=\lambda_1<\lambda_2\le\lambda_3\le\dots\le\lambda_n\) and
the terms \(\|w_i\|_2^{2}\) are non‑negative,the Rayleigh–Ritz energy
\(\mathcal E(\bar V)=\sum_{i=2}^{n}\lambda_i\|w_i\|_2^{2}\)
is minimised, subject to the norm constraint
\(\sum_{i=2}^{n}\|w_i\|_2^{2}=n\),
by allocating the entire norm budget \(n\) to the smallest eigenvalue

\(\lambda_2\):
\[
  \|w_2\|_2^{2}=n,
  \qquad
  \|w_i\|_2^{2}=0
  \ \ (i\ge3).
\]
Setting \(w_2=\sqrt n\,q\) with any unit vector
\(q\in\mathbb{R}^{m}\) realises that allocation and yields
\[
  \mathcal E(\bar V_{\min})
  = \lambda_2\,\|w_2\|_2^{2}
  = n\,\lambda_2.
\]

\textbf{Step4: Uniqueness of the minimiser.}
If a feasible set \(\{w_i\}_{i=2}^{n}\) achieves the same value
\(n\lambda_2\), then necessarily
\(\sum_{i=2}^{n}(\lambda_i-\lambda_2)\|w_i\|_2^{2}=0\); because each
summand is non‑negative and \(\lambda_i>\lambda_2\) for \(i\ge3\), this
forces \(\|w_i\|_2^{2}=0\) for all \(i\ge3\).
Hence every minimiser must indeed place all energy in the Fiedler mode.
\end{proof}

\appendix

\makeatletter
\@ifundefined{isChecklistMainFile}{
  \newif\ifreproStandalone
  \reproStandalonetrue
}{
  \newif\ifreproStandalone
  \reproStandalonefalse
}
\makeatother

\ifreproStandalone
\documentclass[letterpaper]{article}
\usepackage[submission]{aaai2026}
\setlength{\pdfpagewidth}{8.5in}
\setlength{\pdfpageheight}{11in}
\usepackage{times}
\usepackage{helvet}
\usepackage{courier}
\usepackage{xcolor}
\frenchspacing

\begin{document}
\fi
\setlength{\leftmargini}{20pt}
\makeatletter\def\@listi{\leftmargin\leftmargini \topsep .5em \parsep .5em \itemsep .5em}
\def\@listii{\leftmargin\leftmarginii \labelwidth\leftmarginii \advance\labelwidth-\labelsep \topsep .4em \parsep .4em \itemsep .4em}
\def\@listiii{\leftmargin\leftmarginiii \labelwidth\leftmarginiii \advance\labelwidth-\labelsep \topsep .4em \parsep .4em \itemsep .4em}\makeatother

\setcounter{secnumdepth}{0}
\renewcommand\thesubsection{\arabic{subsection}}
\renewcommand\labelenumi{\thesubsection.\arabic{enumi}}

\newcounter{checksubsection}
\newcounter{checkitem}[checksubsection]

\newcommand{\checksubsection}[1]{%
  \refstepcounter{checksubsection}%
  \paragraph{\arabic{checksubsection}. #1}%
  \setcounter{checkitem}{0}%
}

\newcommand{\checkitem}{%
  \refstepcounter{checkitem}%
  \item[\arabic{checksubsection}.\arabic{checkitem}.]%
}
\newcommand{\question}[2]{\normalcolor\checkitem #1 #2 \color{blue}}
\newcommand{\ifyespoints}[1]{\makebox[0pt][l]{\hspace{-15pt}\normalcolor #1}}

\end{document}
\fi
\end{document}